\documentclass[10pt,twocolumn,letterpaper]{article}

\usepackage[pagenumbers]{cvpr} %

\usepackage{graphicx}
\usepackage{multirow}
\usepackage{amsmath}
\usepackage{amsthm}
\usepackage{xcolor}
\usepackage{todonotes}
\usepackage{booktabs}  %
\usepackage[separate-uncertainty=true]{siunitx}
\DeclareSIUnit{\million}{\text{M}}
\usepackage{pifont}
\usepackage{placeins}

\newtheorem{proposition}{Proposition}
\newtheorem{remark}{Remark}

\definecolor{Mypurple}{HTML}{7C1D6F}
\definecolor{Mymagenta}{HTML}{DC3977}
\definecolor{Myorange}{HTML}{F0746E}
\definecolor{Myyellow}{HTML}{FCDE9C}
\definecolor{Mygreen}{HTML}{7CCBA2}
\definecolor{Myteal}{HTML}{089099}
\definecolor{Myblue}{HTML}{045275}

\newcommand{\minisec}[2]{\par\smallskip\noindent\textbf{#1.} #2}

\newcommand{\cuda}{\textsc{CUDA} }

\newcommand{\figref}[1]{fig.\,\ref{fig:#1}}
\newcommand{\Figref}[1]{Fig.\,\ref{fig:#1}}
\newcommand{\tabref}[1]{tab.\,\ref{tab:#1}}
\newcommand{\Tabref}[1]{Tab.\,\ref{tab:#1}}
\newcommand{\secref}[1]{sec.\,\ref{sec:#1}}

\newcommand{\appref}[1]{appendix\,\ref{appendix:#1}}

\newcommand{\pointcloud}{$\{ \bfx_i \}_{i=0}^N$ }
\newcommand{\bfu}{\mathbf{u}}
\newcommand{\bfh}{\mathbf{h}}
\newcommand{\bfx}{\mathbf{x}}
\newcommand{\bfy}{\mathbf{y}}
\newcommand{\bfPhi}{\mathbf{\Phi}}
\newcommand{\bfA}{\mathbf{A}}
\newcommand{\bfB}{\mathbf{B}}
\newcommand{\bfC}{\mathbf{C}}
\newcommand{\bfa}{\mathbf{a}}
\newcommand{\bfb}{\mathbf{b}}
\newcommand{\bfc}{\mathbf{c}}

\newcommand{\linear}[1]{\mathrm{Linear}\left(#1\right)}
\newcommand{\softplus}[1]{\mathbf{\Psi}\left(#1\right)}
\newcommand{\func}[1]{\left(#1\right)}
\newcommand{\R}[1]{\mathbb{R}^{#1}}
\newcommand{\C}[1]{\mathbb{C}^{#1}}
\newcommand{\Order}[1]{\mathcal{O}\left(#1\right)}

\newcommand*\samethanks[1][\value{footnote}]{\footnotemark[#1]}

\definecolor{cvprblue}{rgb}{0.21,0.49,0.74}
\usepackage[pagebackref,breaklinks,colorlinks,allcolors=cvprblue]{hyperref}

\def\paperID{16206} %
\def\confName{CVPR}
\def\confYear{2025}

\title{STREAM: A Universal State-Space Model for Sparse Geometric Data}
\author{
Mark Sch\"one\textsuperscript{1\thanks{These authors contributed equally to this work.}} \and
Yash Bhisikar\textsuperscript{2\samethanks} \and
Karan Bania\textsuperscript{2\samethanks} \and
Khaleelulla Khan Nazeer\textsuperscript{1} \and
Christian Mayr\textsuperscript{1,3,4} \and
Anand Subramoney\textsuperscript{5} \and
David Kappel\textsuperscript{6}
\\\rule{0pt}{3ex}
\textsuperscript{1}TUD Dresden University of Technology\quad%
\textsuperscript{2}BITS Pilani\\ %
\textsuperscript{3}Center for Scalable Data Analytics and Artificial Intelligence (ScaDS.AI)\\
\textsuperscript{4}Centre for Tactile Internet (CeTI) with Human-in-the-Loop\\
\textsuperscript{5}Royal Holloway, University of London\quad %
\textsuperscript{6}Bielefeld University
\\\rule{0pt}{3ex}
{\tt\small \{mark.schoene, khaleelulla.khan, christian.mayr\}@tu-dresden.de}\\
{\tt\small \{f20210483, f20212582\}@goa.bits-pilani.ac.in}\\
{\tt\small anand.subramoney@rhul.ac.uk}\\
{\tt\small david.kappel@uni-bielefeld.de}
}

\begin{document}
\maketitle

\begin{abstract}
Handling sparse and unstructured geometric data, such as point clouds or event-based vision, is a pressing challenge in the field of machine vision. 
Recently, sequence models such as Transformers and state-space models entered the domain of geometric data.
These methods require specialized preprocessing to create a sequential view of a set of points.
Furthermore, prior works involving sequence models iterate geometric data with either uniform or learned step sizes,
implicitly relying on the model to infer the underlying geometric structure.
In this work, we propose to encode geometric structure explicitly into the parameterization of a state-space model.
State-space models are based on linear dynamics governed by a one-dimensional variable such as time or a spatial coordinate.
We exploit this dynamic variable to inject relative differences of coordinates into the step size of the state-space model.
The resulting geometric operation computes interactions between all pairs of $N$ points in $\Order{N}$ steps.
Our model deploys the Mamba selective state-space model with a modified \cuda kernel to efficiently map sparse geometric data to modern hardware.
The resulting sequence model, which we call STREAM, achieves competitive results on a range of benchmarks from point-cloud classification to event-based vision and audio classification.
STREAM demonstrates a powerful inductive bias for sparse geometric data by improving the PointMamba baseline  when trained from scratch on the ModelNet40 and ScanObjectNN point cloud analysis datasets.
It further achieves, for the first time, \SI{100}{\percent} test accuracy on all 11 classes of the DVS128 Gestures dataset.

\end{abstract}
\section{Introduction}
\label{sec:intro}
\begin{figure}
    \centering
    \includegraphics[width=\linewidth]{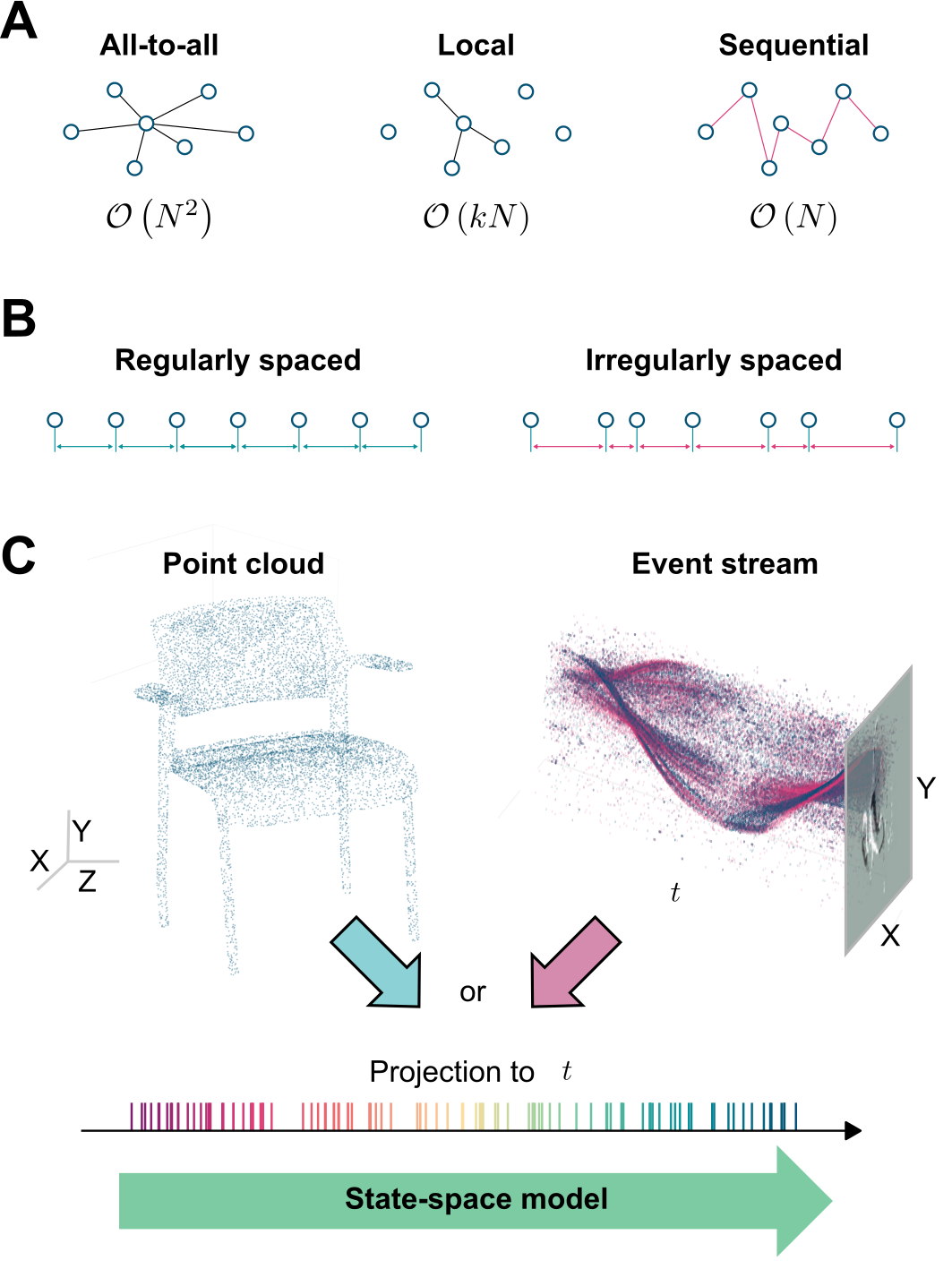}
    \caption{
        A unified view on point cloud and event stream modeling.
        \textbf{A.} Adjacency structure of point based methods 
        \textbf{B.} The coordinates of sparse geometric data are \emph{irregularly spaced}
        \textbf{C.} Point clouds and event streams are ordered by spatial and temporal axes, respectively, before the state-space model.
    }
    \label{fig:main}
\end{figure}
Computer vision computes relationships between data points in spatial coordinates $X$, $Y$ (in $\mathbb{R}^2$) or $X$, $Y$, $Z$ (in $\mathbb{R}^3$), or spatio-temporal coordinates ($\mathbb{R}^3$ or $\mathbb{R}^4$), where one of the dimensions denotes time $t$.
Convolutional neural networks based on structured, uniformly spaced and local linear operations successfully address this problem for classical camera recordings such as images or videos, which are themselves structured and discrete.
Many modalities of recent interest, however, are neither structured nor uniformly spaced. 
Sensors such as Light Detection and Ranging (LiDAR) or event-based cameras~\cite{Lichtsteiner2008, PropheseeGen4} sample signals based on sparse  processes, resulting in sparse geometric data with irregularly spaced coordinates.
Point cloud analysis was the central research field for sparse geometric data over the past decade~\cite{Guo2021pointcloudsurvey}.
While early works followed voxel-based approaches~\cite{Wu2015shapenets, Maturana2015voxnet}, point-based methods dominate the research landscape today~\cite{qi2017pointnetdeeplearningpoint, Wang2018, Wu2019, pointgpt23}.
Meanwhile, event-based cameras~\cite{Lichtsteiner2008, PropheseeGen4} raised considerable interest in the computer vision community.
Although these cameras record sparse geometric data, most works collapse the sparse stream of events into frames~\cite{ExACT2024, Zubic2024, PropheseeGen4}, potentially losing unique properties of these cameras such as low latency and high dynamic range.
The structural similarity between point clouds and event streams encourages methodological transfer, especially from the domain of point cloud analysis to event-based vision~\cite{Wang2019eventpointnet, Ren2023ttpoint}.
Our work demonstrates the reverse: The inherently temporal state-space model formulation presented in \secref{modeling-sparse-unstructured-coordinates} improves event-based vision on the DVS128-Gesture dataset~\cite{Amir2017} to \SI{100}{\percent}, and at the same time provides a valuable inductive bias for geometric data such as point clouds as demonstrated on the ModelNet40 and ScanObjectNN datasets~\cite{modelnet40, scanobjectnn}.

We show that sparse geometric data with irregularly spaced coordinates as shown in \figref{main}B can be naturally integrated in the state-space model formalism.
In contrast to prior sequence models such as PointMamba~\cite{liang2024pointmambasimplestatespace} or EventMamba~\cite{ren2024rethinkingefficienteffectivepointbased}, we \emph{explicitly} encode geometric structure into the state-space model parameterization by injecting relative differences of coordinates as step sizes.
\minisec{Contributions}
This paper makes the following main innovations over the state-of-the-art:
\begin{itemize}
    \item We propose a unified framework for modeling sparse geometric data with state-space models with irregularly spaced step sizes.
The resulting model, STREAM, handles sparse geometric data such as point clouds \emph{and} event-based vision.
    
    \item Our state-space model formulation demonstrates a valuable inductive bias for geometric structure by improving the PointMamba~\cite{liang2024pointmambasimplestatespace} baseline by up to \SI{2.84}{\percent} when trained from scratch on the ScanObjectNN dataset~\cite{scanobjectnn}.
    \item We show compelling results for event-based vision, processed as a stream of events with our purely recurrent neural network without the usage of frames or spatial convolutions.
    For the first time, we demonstrate \SI{100}{\percent} classification accuracy on all 11 classes of the DVS128-Gestures dataset \cite{Amir2017}.
    \item We provide an efficient \cuda implementation for integrating irregular step sizes based on the selective-scan implementation of \cite{gu2024mambalineartimesequencemodeling}.
\end{itemize}
\section{Related Work}
\label{sec:related-work}
\subsection{Point Cloud Analysis}
\label{sec:related-work/points}
Early point-based methods such as PointNets~\cite{qi2017pointnetdeeplearningpoint, pointnet++} directly pass the point cloud through a Multi-layer Perceptron (MLP) and introduce a permutation invariance of the set of points via pooling operations, \emph{implicitly} integrating spatial information.
In contrast, methods for \emph{explicitly} integrating spatial information parameterize linear operators such as convolutions based on relative differences between point coordinates.
The convolutional neural networks based on irregularly spaced operators presented in \cite{Wang2018, Wu2019, Wu2023, Kim2023} outperform PointNets on a range of point cloud analysis tasks.
These works parameterize the convolution kernel with MLPs evaluated at the relative differences between point coordinates.
To break the $\Order{N^2}$ complexity of computing pairwise interactions for all $N$ points, 
locality constrains are added to scale to larger point sets (see \figref{main}A).
As shown in \figref{method}A, our method also explicitly integrates spatial information in this sense, but with a kernel parameterized by a state-space model (SSM) instead of a MLP (see equation \eqref{eq:integral-operator-dirac}).
This parameterization enables the complete computation to be performed in $\Order{N}$ steps.

More recently, sequence models demonstrated favorable results over convolutional methods~\cite{pointgpt23, yu2022pointbertpretraining3dpoint, pang2022maskedautoencoderspointcloud, liang2024pointmambasimplestatespace}.
The point cloud is flattened into a sequence and then processed by transformer or state-space model based backbones.
This formulation inherits many strong scaling properties of sequence models including masked pre-training~\cite{yu2022pointbertpretraining3dpoint} or generative pre-training~\cite{pointgpt23} objectives.
Transformers, however, suffer from $\Order{N^2}$ complexity in the number of points $N$, limiting their application to larger point clouds.
State-space models, in contrast, have sequential $\Order{N}$ complexity for inference and $\Order{N\log N}$ complexity for parallel training~\cite{Gu2021, Gu2022}.
At the same time, SSMs can learn long distant relationships between inputs including convolutional operators with rational transfer functions~\cite{Gu2021, Poli2023}.
The recently introduced Mamba state-space model~\cite{gu2024mambalineartimesequencemodeling} sparked widespread discourse, and was quickly adapted in many domains, including point cloud analysis.
Works like PointMamba~\cite{liang2024pointmambasimplestatespace} or Point Cloud Mamba (PCM)~\cite{Tao2024pointcloudmamba} deploy Mamba in established point cloud frameworks such as the masked autoencoder developed in~\cite{pang2022maskedautoencoderspointcloud}.
By relying on the default Mamba architecture, spatial information is integrated \emph{implicitly} in these works similar to PointNets.
In contrast, we show in \secref{modeling-sparse-unstructured-coordinates} that spatial information can be integrated \emph{explicitly} via the SSM parameterization. 
This formulation allows us to remove specialized preprocessing methods to represent point clouds as sequences~\cite{liang2024pointmambasimplestatespace, Tao2024pointcloudmamba}, and simply order by the spatial coordinates $X$, $Y$, or $Z$.

\subsection{Event-based Vision}
\label{sec:related-work/events}
Most state-of-the-art event-based vision methods construct frames from the asynchronous event-stream by binning events in regularly spaced time intervals~\cite{subramoney2023efficient, Zubic2024, PropheseeGen4, ExACT2024, chen2024spikmambasnnmeetsmamba}.
Few methods operate on the sparse unstructured event streams directly.
AEGNN~\cite{Schaefer2022} computes spatio-temporal features based on subsampled event streams with graph neural networks.
Their method admits an event-based inference method where only nodes that correspond to incoming events are updated asynchronously.
EventMamba~\cite{ren2024rethinkingefficienteffectivepointbased} is tightly related to the point cloud analysis models discussed in \secref{related-work/points}.
They process subsampled raw event streams similar to \cite{Schaefer2022, Sekikawa2019} with a combination of a point feature extractor and Mamba.
Its high similarity with PointMamba~\cite{liang2024pointmambasimplestatespace} places EventMamba in the category of methods \emph{implicitly} integrating spatio-temporal information.
Conceptually closer to our model are filtering methods that process the stream of events recursively while \emph{explicitly integrating spatio-temporal information}.
Early methods like HOTS~\cite{Lagorce2017} create a spatial representation of the $(X, Y)$-plane of an event-based camera by integrating exact timestamps with exponentially moving averages. Similar to our method, EventNet~\cite{Sekikawa2019} recursively iterates the event stream. 
In contrast to our state-space model, their recurrent network parameterization does not inherit stability guarantees from the theory of linear systems and lacks an efficient parallelization along the time dimension during training,
limiting the scalability to longer streams.
Event-SSM~\cite{Schoene2024} proposed a discretization method for simplified state-space layers (S5)~\cite{Smith2023} to operate directly on asynchronous event streams.
Their model excels at spiking audio classification, but falls behind the state-of-the-art in event-based vision.
Concurrently to our work, S7~\cite{soydan2024s7selectivesimplifiedstate} explores improvements of the S5 architecture and Event-SSM through a more stable parameterization and input dependent state-space model parameters.
In contrast, STREAM can improve sequence-based point cloud analysis (see \secref{point-clouds}), and therefore drive the convergence between event-based vision and point cloud analysis forward.

Related from the perspective of discretizing continuous kernels is the TENNs framework~\cite{Pei2024}.
They construct convolutions over finite time horizons with orthogonal polynomials that can be discretized on irregularly spaced timestamps. 
In practice, they choose regularly spaced frames, however, and show compelling results on event-based vision benchmarks.
In comparison, STREAM parameterizes infinite time horizons via the state-space model.
\section{STREAM}
\label{sec:methods}
\begin{figure*}[ht]
    \centering
    \includegraphics[width=1.0\textwidth]{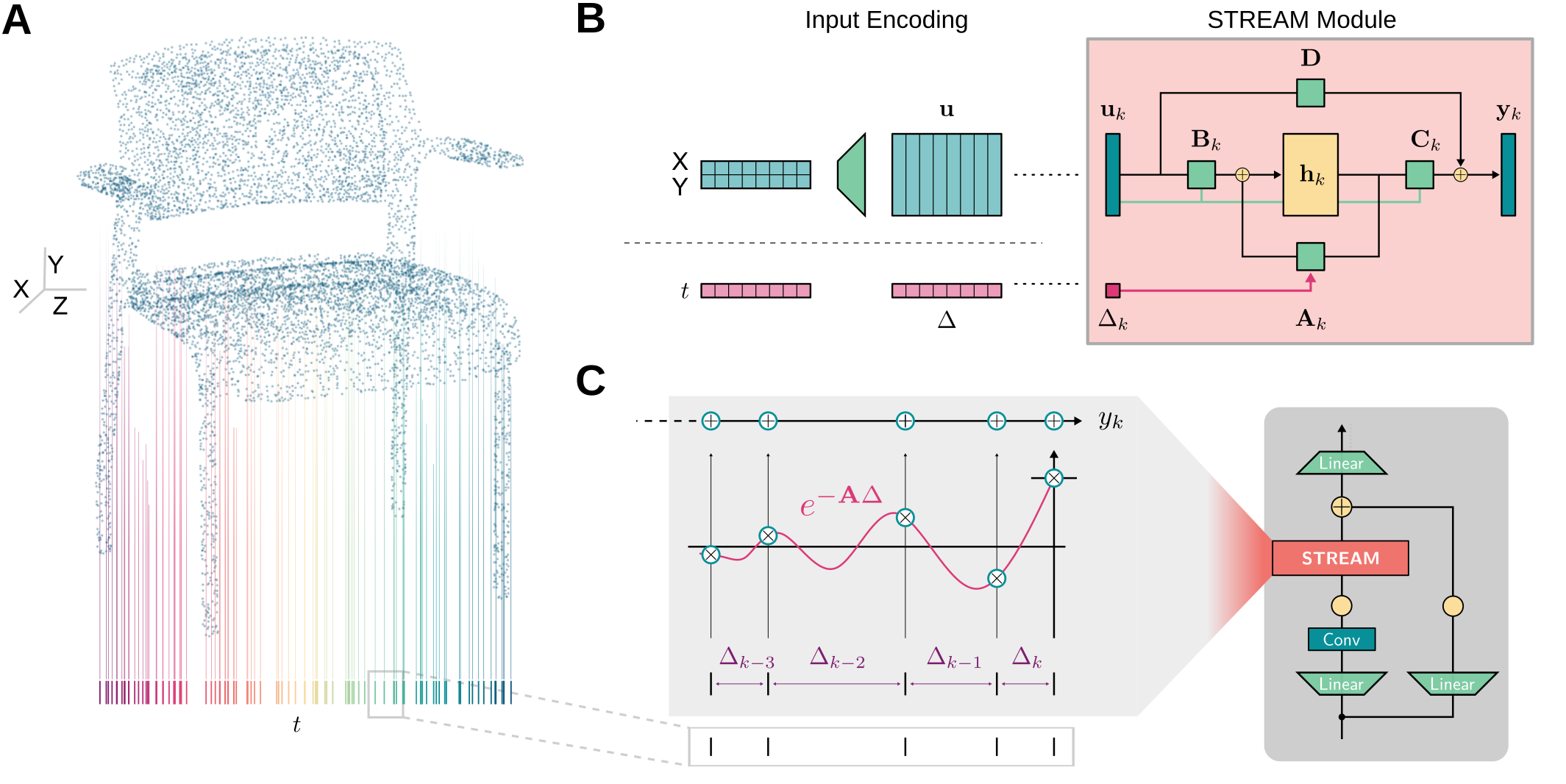}
    \caption{
        STREAM encodes geometric information into the SSM parameters.
        \textbf{A.}~Point cloud input.
        \textbf{B.}~The STREAM module converts relative differences in coordinates into the $\Delta_k$ scale of the SSM.
        \textbf{C.}~A STREAM module integrates pairwise spatial relationships based on an exponentially oscillating kernel.
    }
    \label{fig:method}
\end{figure*}
In this section, we present our method for efficiently modeling sparse sets of $N$ coordinates \pointcloud, where ${\bfx_i = \left( x_i^{(1)}, \dots, x_i^{(d)}\right) \in\R{d}}$.
This general formulation includes data in 3D space, with ${\bfx = \left(X, Y, Z\right)}$, as well as the streams recorded by event-based cameras. In the latter case one of the dimensions (e.g. $x^{(1)}$) is interpreted as event time while others  ($x^{(2)}$, $x^{(3)}$) denote spatial locations (see \figref{main}C). Furthermore, higher dimensions ($d>3$) can be used to denote data that includes other dimensions, such as color intensity, etc.  STREAM is illustrated in \figref{method}A-C.
\minisec{Notation}
We denote tensor-valued variables and functions such as multi-dimensional coordinates $\bfx\in\R{d}$ or matrices $\bfA\in\R{m\times m}$ in bold characters.
Scalar variables such as the coordinates $x^{(1)}, \dots, x^{(d)}$ of a tensor $\bfx$ or the entries $A_{ij}$ of a matrix $\bfA$ are denoted in regular characters.
\subsection{Problem Statement}
\label{sec:problem-statement}
Consider a set \pointcloud of points sparsely distributed in space.
We assign a multi-dimensional representation vector $\bfu_i$ to each coordinate $\bfx_i$ as common practice in machine learning.
Such input signals can be formulated mathematically as sums of Dirac delta pulses evaluated at the coordinates~\pointcloud
\begin{align}
    \bfu\func{\bfx} = \sum_{i=0}^N \delta\func{\bfx - \bfx_i} \bfu_i \,, \label{eq:input}
\end{align}
where $\bfu_i\in\R{n}$ is the representation at $\bfx_i$.
To reason about spatial relationships between these pulses, we define an interaction kernel $\bfPhi(\bfx, \bfx^\prime)$ that models pairwise interactions.
A complete view of all possible interactions with the point $\bfx_k$ is given by 
\begin{align}
    \bfy\func{\bfx_k} 
    &= \int \bfPhi\func{\bfx_k, \bfx^\prime} \bfu\func{\bfx^\prime} \mathrm{d}\bfx^\prime \label{eq:integral-operator} \\
    &= \sum_{i=0}^N \bfPhi\func{\bfx_k, \bfx_i} \bfu_i\,. \label{eq:integral-operator-dirac}
\end{align}
This formulation contains the familiar convolution operator as a special case when $\bfPhi\func{\bfx, \bfx^\prime} \equiv\bfPhi\func{\bfx - \bfx^\prime}$.
Equation \eqref{eq:integral-operator-dirac} highlights that the computation of $\bfy\func{\bfx_i}$ at the coordinates \pointcloud requires evaluating $\bfPhi\func{\bfx_i, \bfx_j}$ only at pairs of these coordinates.
In the general case, this operation requires $\Order{N^2}$ steps to compute.
As shown in \figref{main}A, this complexity can be reduced to $\Order{kN}$ by focusing on local neighborhoods of $k$ points,
a technique broadly used in the literature of point convolutions \cite{Wang2018, Wu2019}.
In the following sections, we present a method for recursively integrating $N$ sparsely sampled points in $\Order{N}$ time without restricting to local neighborhoods.
\subsection{Integrating Sparse Coordinates with State-space Models}
\label{sec:modeling-sparse-unstructured-coordinates}

State-space models (SSMs) have demonstrated outstanding performance on long-range dependency modeling tasks, 
while remaining efficient for training and inference~\cite{Gu2022, gu2024mambalineartimesequencemodeling}.
SSMs project the input through a linear recursive kernel to states $\bfh\func{t}$.
The independent variable $t$ is usually interpreted as, and aligned with, time in the input signal domain.
The linear nature of SSMs allows the exact integration of sparse sets of points with irregularly spaced coordinates, as we will show below.
A linear time-varying SSM acting on a scalar input function $u\func{t}$ and producing a scalar output $y\func{t}$ is defined through
\begin{align}
    \dot{\bfh}\func{t} &= \bfA\func{t} \bfh\func{t} + \bfB\func{t} u\func{t} \label{eq:cont-ssm-h}\\
    y\func{t} &= \bfC\func{t} \bfh\func{t}\,, \label{eq:cont-ssm-y}
\end{align}
with ${u(t), y(t)\in\mathbb{R}}$, states ${\bfh\func{t}\in\R{m}}$ and parameters ${\bfA\func{t}\in\R{m\times m}}, {\bfB\func{t}\in\R{m\times 1}}, {\bfC\func{t}\in\R{1\times m}}$.
For the sake of simplicity, we assume that the parameters are represented by piecewise constant functions, i.e.
$\bfA_i = \bfA\func{t_i}$, $\bfB_i = \bfB\func{t_i}$, $\bfC_i = \bfC\func{t_i}$ are constant on the intervals $(t_{i-1}, t_i]$.
We show in \appref{ssm} that solving the linear system in equations~\eqref{eq:cont-ssm-h} and~\eqref{eq:cont-ssm-y} for the input in equation~\eqref{eq:input} yields the kernel function
\begin{align}
    \bfPhi\func{t_k, t_i} = \bfC_k \left(\prod_{j=i+1}^k\exp\func{\bfA_j\Delta_j}\right) \bfB_i \,, \label{eq:kernel}
\end{align}
where $\Delta_j = t_j - t_{j-1}$.
Importantly, the output $y(t)$ of the state-space model at $t_0, \dots, t_N$ can be computed with the recursive formula
\begin{align}
    \bfh_k = \bfh\func{t_k} &= e^{\bfA_k \Delta_k} \bfh_{k-1} + \bfB_k u_k \label{eq:recurrent-ssm} \\
    y_k  = y(t_k) &= \bfC_k \bfh_k \,.
\end{align}
Equations \eqref{eq:kernel} and \eqref{eq:recurrent-ssm} highlight how relative differences $\Delta_j = t_j - t_{j-1}$ in the coordinate $t$ explicitly parameterize our pairwise interaction kernel $\bfPhi$.
An example of the interaction defined by equation \eqref{eq:kernel} can be found in \Figref{method}C.
For complex $\bfA\in\C{}$, the exponentially oscillating kernel \(e^{\bfA\Delta}\) integrates the history of signals $u_i$ occuring in irregularly spaced intervals $\Delta_i$ with $i < k$.
With this parameterization, our method differs from other Mamba based works, where $\Delta$ is a function of $u_i$ that doesn't explicitly take the dynamics of $t$ into account.
A complete derivation can be found in \appref{ssm}.

The recursive formulation in equation \eqref{eq:recurrent-ssm} is a linear recurrent neural network, whose state-to-state transition matrix \(e^{\bfA_k \Delta_k}\) is parameterized by the differences ${\Delta_k = t_k - t_{k-1}}$ in the irregularly spaced coordinates $t_i$.
As such, our method positions itself among the methods discussed in \secref{related-work/points} that \emph{explicitly integrate spatial information} through the parameterization of a linear operator with the coordinates.

The output sequence $y\func{t_0}, \dots, y\func{t_N}$ can be computed sequentially in $\Order{N}$ time.
This allows \emph{efficient inference on streams} of pulses recorded by sensors such as event-based cameras or LiDAR.
At the same time, linear recursive equations admit an \emph{efficient parallelization} via the \texttt{Scan} primitive available in \cuda \cite{blelloch1990prefix}.
On sufficiently many processors, the output sequence can be computed in $\Order{\log N}$ time, which allows massive parallelization on very long sequences of up to 1~million inputs \cite{gu2024mambalineartimesequencemodeling}.
We show in \appref{scan} how our formulation maps to the scan primitive.
\subsection{Selective state-space models and STREAM}
\label{sec:STREAM}
We develop our method based on the selective state-space model also known as Mamba~\cite{gu2024mambalineartimesequencemodeling}, which was designed for regularly spaced modalities, and language modeling in particular.
In accordance with this purpose, the integration time steps $\Delta_i$ and the matrices $\bfB_i, \bfC_i$ are functions of the input signal $u_i$, and do not carry explicit information about spatial relationships.
Formally, Mamba is parameterized by
\begin{align}
    \bfA_i &\equiv \bfA \quad \forall i &\Delta_i &= \softplus{\linear{u_i}} \label{eq:mamba1}\\
    \bfB_i &= \Delta_i\linear{u_i} & \bfC_i &= \linear{u_i} \,, \label{eq:mamba2}
\end{align}
where $\bfA$ is a learned diagonal matrix, and $\softplus{u} = \ln(1 + e^u)$ is the softplus function.
This single-input $u_i$ single-output $y_i$ SSM is then repeated $n$ times and wrapped by linear transformations to create a multi-input $\bfu_i$ multi-output $\bfy_i$ model as shown in \figref{method}B.
Note that the total state size is $nm$ in the multi-input multi-output case.

We propose to \emph{explicitly} represent the irregularly spaced pulse timings $t_i$ in the Mamba parameterization according to equation \eqref{eq:recurrent-ssm}.
Therefore, we explicitly set {\color{Mypurple}{$\Delta_i$}} to $\func{t_i - t_{i-1}}\softplus{\delta}$, where $\delta$ is a trainable time scale parameter, 
in contrast to equation \eqref{eq:mamba1}.
We further decouple the coefficient of $\bfB_i$ in equation \eqref{eq:mamba2} from the time differences {\color{Mypurple}{$\Delta_i$}} to avoid zero coefficients in case of overlapping pulses ${t_i - t_{i-1} = 0}$.
With these adjustments, our model is defined by
\begin{align}
    \bfA_i &\equiv \bfA \, \forall i & {\color{Mypurple}{\Delta_i}} &= \func{t_i - t_{i-1}} \softplus{\delta}\label{eq:stream1}\\
    \bfB_i &= {\color{Myteal}{\Gamma_i}} \linear{u_i} & \bfC_i &= \linear{u_i} \label{eq:stream2}\\
    {\color{Myteal}{\Gamma_i}} &= \softplus{\linear{u_i}}\,,
\end{align}
where $\delta\in\R{n}$ is a learnable parameter.
We call our resulting state-space model parameterization STREAM for Spatio-Temporal Recursive Encoding of Asynchronous Modalities.
\subsection{Modeling Point Clouds with STREAM}
\label{sec:point-clouds-stream}
The explicit coordinate $t$ in the STREAM module is one-dimensional and strictly ordered.
For the case of event streams, this ordering corresponds to temporal order from past to future.
Point clouds fit into this framework by ordering the set of points w.r.t. one of the spatial coordinates from left to right.
As shown in \figref{method}A,B, we select one of the coordinates to be integrated explicitly.
Inserting the $X$ coordinate for $t$ in equation \eqref{eq:stream1} without loss of generality and setting $\bfu_k = \mathbf{\Theta}\func{X_k, Y_k, Z_k}$, where ${\mathbf{\Theta}:\R{3}\longrightarrow\R{n}}$ is a point wise encoder, yields our STREAM formulation for point cloud analysis.
To achieve state-of-the-art results, we sort the point set by all three spatial coordinates respectively and feed the concatenation of the three resulting sequences to STREAM (see \secref{point-clouds-model}, \figref{point-clouds}).
\section{Experiments}
\label{sec:experiments}
We demonstrate empirically that the explicitly parameterization of the state-space model with coordinate differences $\Delta_i$ presented in sections~\ref{sec:modeling-sparse-unstructured-coordinates} and~\ref{sec:STREAM} is a strong abstraction that exploits the sparse geometry of both point clouds and event-based vision.
When trained from scratch, our method improves the PointMamba~\cite{liang2024pointmambasimplestatespace} baseline by up to \SI{2.8}{\percent} on the ScanObjectNN dataset~\cite{scanobjectnn}.
At the same time, our method sets a new state-of-the-art for event-based gesture recognition on the DVS128-Gestures dataset~\cite{Amir2017}.
\subsection{STREAM for Point Cloud Analysis}
\label{sec:point-clouds}
\subsubsection{Point Cloud Model}
\label{sec:point-clouds-model}
\begin{figure*}[ht]
    \centering
    \includegraphics[width=\textwidth]{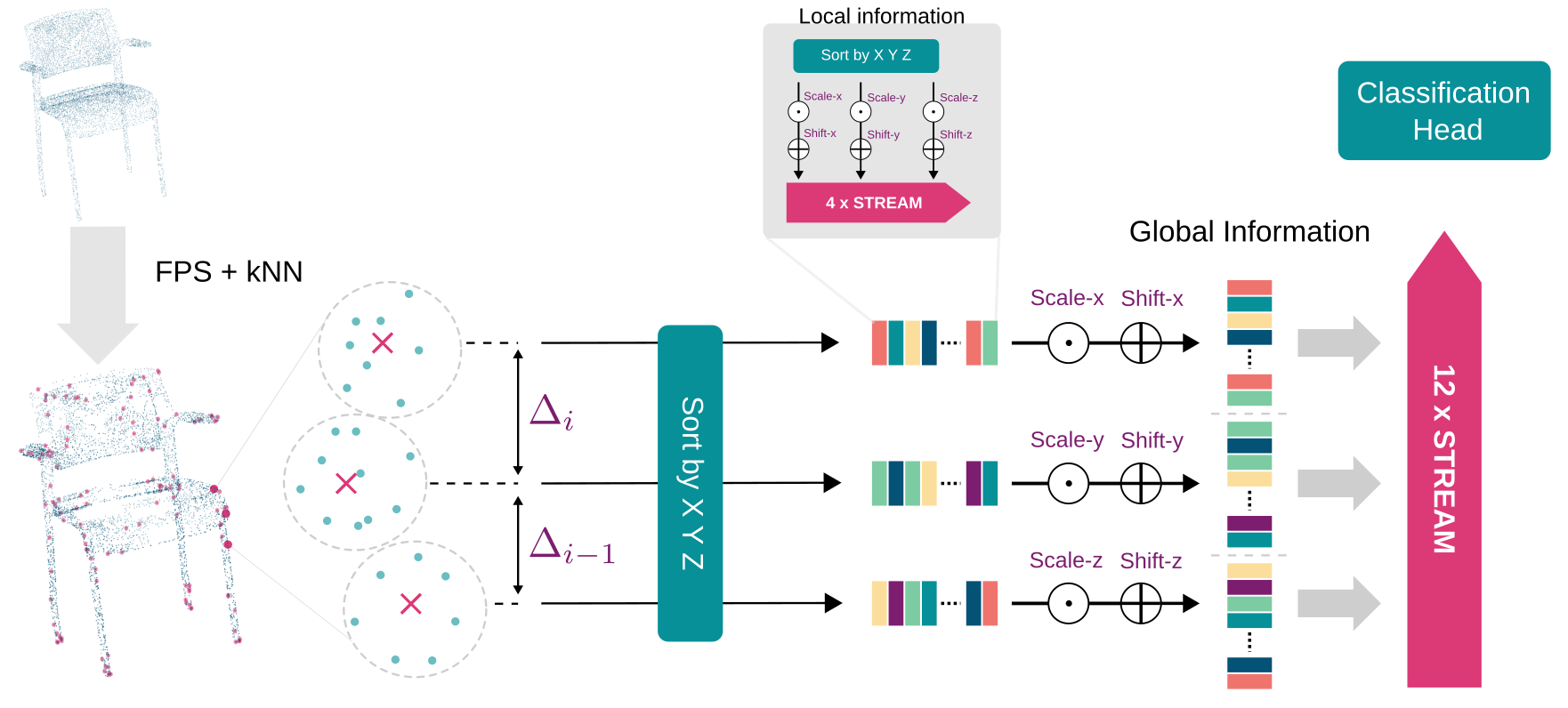}
    \caption{Point cloud processing pipeline inspired by PointMamba~\cite{liang2024pointmambasimplestatespace}.
    FPS: Farthest point sampling, kNN: k-Nearest Neighbors.}
    \label{fig:point-clouds}
\end{figure*}
Our goal is to evaluate the explicit SSM parameterization defined in \secref{STREAM}.
Therefore, we align our architecture close to PointMamba~\cite{liang2024pointmambasimplestatespace}, which uses the default Mamba module~\cite{gu2024mambalineartimesequencemodeling}.
As visualized in \figref{point-clouds}A,
we sample $N$ points and sort this point set by the $X$, $Y$, and $Z$ coordinates respectively.
The vector representations extracted from the point coordinates individually are scaled and shifted by learned parameters for the three sorted sequences separately to indicate the sorting dimension.
The resulting sequence of $3N$ points is concatenated along the sequence dimension to form the input to our model.
This preprocessing simplifies the ordering of the point set by avoiding the two space-filling Hilbert-curves applied by PointMamba at the cost of increasing the input sequence length from $2N$ to $3N$.

Similar to \cite{pang2022maskedautoencoderspointcloud, pointgpt23, liang2024pointmambasimplestatespace}, our model is composed of two stages.
The input sequence is grouped into sets of \num{32}~points that are independently processed by a STREAM module to represent local information.
A backbone of \num{12}~STREAM layers operates on the features extracted from these groups to agglomerate global information.
\subsubsection{Point Cloud Results}
\begin{table*}[htb!]
    \centering
    \begin{tabular}{lcc S[table-format=2.2] S[table-format=2.2] S[table-format=2.2] S[table-format=2.2]}
        \toprule
        \textbf{Methods} & \textbf{Backbone} & \textbf{Param. (M) $\downarrow$} & \textbf{OBJ-BG $\uparrow$} & \textbf{OBJ-ONLY $\uparrow$} & \textbf{PB-T50-RS $\uparrow$} \\ 
        \midrule
        \multicolumn{6}{c}{Supervised Learning Only} \\ 
        \midrule
        PointMLP~\citep{pointmlp} & {-} & 13.2 & 85.4 & {-} & 85.4 \\ 
        RepSurf-U~\citep{ran2022surfacerepresentationpointclouds} & {-} & 1.5 & {-} & {-} & 84.3 \\ 
        ADS~\citep{ads_hong23} & {-} & {-} & {-} & {-} & 87.5 \\ 
        Point-MAE~\citep{pang2022maskedautoencoderspointcloud} & Transformer & 22.1 & 86.75 & 86.92 & 80.78 \\
        PCM~\cite{Tao2024pointcloudmamba} & {Mamba} & 34.2 & {-} & {-} & 88.1 \\
        PointMamba~\citep{liang2024pointmambasimplestatespace} & Mamba & 12.3 & 88.30 & 87.78 & 82.48 \\
        Ours & STREAM & 12.3 & 90.02~{\small \color{Mymagenta}{(+\num{1.72})}} & 88.64~{\small \color{Mymagenta}{(+\num{0.86})}} & 85.32~{\small \color{Mymagenta}{(+\num{2.84})}} \\
        \midrule
        \multicolumn{6}{c}{Training from Pre-training (Single-Modal)} \\ 
        \midrule
        Point-BERT~\cite{yu2022pointbertpretraining3dpoint} & Transformer & 22.1 & 87.43 & 88.12 & 83.07 \\ 
        MaskPoint~\citep{liu2022maskeddiscriminationselfsupervisedlearning} & Transformer & 22.1 & 89.30 & 88.10 & 84.30 \\ 
        Point-MAE~\citep{pang2022maskedautoencoderspointcloud} & Transformer & 22.1 & 92.77 & 91.22 & 89.04 \\ 
        PointMamba~\citep{liang2024pointmambasimplestatespace} & Mamba & 12.3 & 94.32 & 92.60 & 89.31 \\ 
        \bottomrule
    \end{tabular}
    \caption{
        Object classification on ScanObjectNN~\cite{scanobjectnn}.
        We report overall accuracy (OA) in percent and
        indicate the {\color{Mymagenta}{improvement}} over the PointMamba baseline through using our STREAM module instead of the default Mamba module.
        }
    \label{tab:scanobjectnn}
\end{table*}
We evaluate our method on two standardized point cloud classification benchmarks.
The ModelNet40 dataset~\cite{modelnet40} contains \num{12311} clean point clouds sampled from 3D CAD objects.
A more realistic dataset of physically scanned indoor scenes is given by ScanObjectNN~\cite{scanobjectnn}.
ScanObjectNN contains about \num{15000} scans from \num{2902} unique objects with varying difficulty.
In both cases, we train our method with the same hyperparameters as PointMamba~\cite{liang2024pointmambasimplestatespace}.

We report the best overall classification accuracy, i.e. the average over all samples, obtained from 5 different random seeds.
\Tabref{scanobjectnn} shows that replacing the Mamba module in PointMamba with a STREAM module improves the performance of models trained from scratch across all instances on ScanObjectNN.
In particular, the hardest instance PB-T50-RS benefits from the exact integration of spatial information and improves PointMamba from \SI{82.48}{\percent} to \SI{85.32}{\percent} for the best out of 5 random seeds.
With a mean and standard deviation of \SI{84.4+-0.7}{\percent}, this poses a significant improvement.

On ModelNet40, STREAM improves the overall accuracy of PointMamba by \SI{0.3}{\percent} from \SI{92.4}{\percent} to \SI{92.7}{\percent} as reported in \tabref{modelnet40}.
Despite the smaller gain compared to ScanObjectNN, the low variance on this dataset indicates a significant improvement on ModelNet40 as well.
The mean and standard deviation of our model is \SI{92.6+-0.1}{\percent}.

Our results show that the inductive bias of explicitly encoding spatial relationships in the SSM parameterization is particularly useful when training models from scratch.
\begin{table}[htb!]
    \centering
    \begin{tabular}{l S[table-format=2.2] S[table-format=2.2]}
        \toprule
        \textbf{Methods} & \textbf{Param. (M) $\downarrow$} & \textbf{OA (\%) $\uparrow$} \\ 
        \midrule
        \multicolumn{3}{c}{Supervised Learning Only} \\ 
        \midrule
        PointNet~\citep{qi2017pointnetdeeplearningpoint} & 3.5 & 89.2 \\ 
        PointNet++~\citep{pointnet++} & 1.5 & 90.7 \\ 
        PointCNN~\citep{pointcnn18} & 0.6 & 92.2 \\ 
        DGCNN~\citep{dgcnn18} & 1.8 & 92.9 \\ 
        PointNeXt~\citep{qian2022pointnextrevisitingpointnetimproved} & 1.4 & 92.9 \\ 
        PCT~\citep{PCT21} & 2.9 & 93.2 \\ 
        OctFormer~\citep{octformer23} & 4.0 & 92.7 \\
        PCM-Tiny~\cite{Tao2024pointcloudmamba} & 6.9 & 93.1 \\
        Point-MAE~\citep{pang2022maskedautoencoderspointcloud} & 22.1 & 92.3 \\
        PointMamba~\citep{liang2024pointmambasimplestatespace} & 12.3 & 92.4 \\
        STREAM (ours) & 12.3 & 92.7~{\small \color{Mymagenta}{(+\num{0.3})}} \\
        \midrule
        \multicolumn{3}{c}{with Self-supervised pre-training} \\ 
        \midrule
        Point-BERT~\citep{yu2022pointbertpretraining3dpoint} & 22.1 & 92.7 \\ 
        MaskPoint~\citep{liu2022maskeddiscriminationselfsupervisedlearning} & 22.1 & 92.6 \\ 
        Point-M2AE~\citep{zhang2022pointm2aemultiscalemaskedautoencoders} & 12.8 & 93.4 \\ 
        Point-MAE~\citep{pang2022maskedautoencoderspointcloud} & 22.1 & 93.2 \\ 
        PointGPT-S~\citep{pointgpt23} & 29.2 & 93.3 \\ 
        ACT~\citep{dong2023autoencoderscrossmodalteacherspretrained} & 22.1 & 93.6 \\ 
        PointMamba~\citep{liang2024pointmambasimplestatespace} & 12.3 & 93.6 \\ 
        \bottomrule
    \end{tabular}
    \caption{
        Classification on ModelNet40. 
        We report overall accuracy (OA) in percent and
        indicate the {\color{Mymagenta}{improvement}} over the PointMamba baseline through using our STREAM module instead of the default Mamba module.
        }
    \label{tab:modelnet40}
\end{table}
\subsection{STREAM for Event Streams}
\label{sec:event-based-sensors}
\begin{figure}[htb!]
    \centering
    \includegraphics[width=\linewidth]{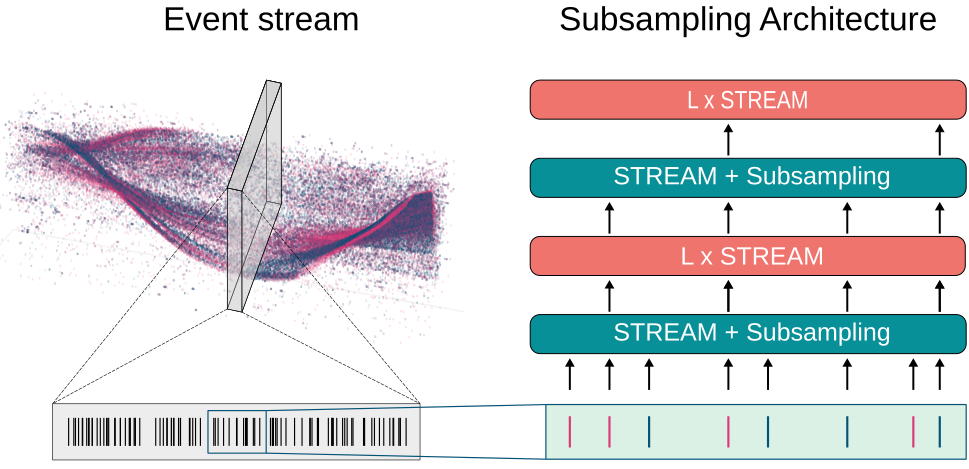}
    \caption{Hierarchical subsampling architectures with two stages to handle long event streams.}
    \label{fig:event-based}
\end{figure}
\subsubsection{Event Stream Model}
\begin{table}[htb!]
    \centering
    \begin{tabular}{l S[table-format=2.1] S[table-format=2.1]}
        \toprule
        \textbf{Methods} & \textbf{Param. (M) $\downarrow$} & \textbf{Acc (\%) $\uparrow$} \\ 
        \midrule
        \multicolumn{3}{c}{Frame-based} \\ 
        \midrule
        ACE-BET~\cite{Liu2022} & {-} & 98.9 \\
        ExACT~\cite{ExACT2024} & {-} & 98.9 \\
        SpikMamba~\cite{chen2024spikmambasnnmeetsmamba} & 0.2 & 99.0 \\
        EventMamba~\cite{ren2024rethinkingefficienteffectivepointbased} & 0.3 & 99.2 \\
        TENNs-PLEIADES~\cite{Pei2024} & 0.2 & 100.0{$^*$} \\
        \midrule
        \multicolumn{3}{c}{Event-based} \\ 
        \midrule
        Event-SSM~\cite{Schoene2024} & 5.0 & 97.7  \\
        S7 \cite{soydan2024s7selectivesimplifiedstate} & 4.1 & 99.2 \\
        STREAM (ours) & 0.2~{+~1.0$^{\dagger}$} & 100.0 \\
        \bottomrule
    \end{tabular}
    \caption{
        Comparison of STREAM to the state-of-the-art on the DVS128-Gestures dataset \cite{Amir2017}. 
        $^*$evaluated on 10 out of 11 classes, omitting the "other" class.
        $^{\dagger}$\SI{0.2}{\million} parameters for STREAM, and \SI{1.0}{\million} parameters for encoding the input stream as tokens.
        }    
    \label{tab:gestures}
\end{table}
\begin{table}[htb!]
    \centering
    \begin{tabular}{lcc}
        \toprule
        \textbf{Methods} & \textbf{Param. (M) $\downarrow$} & \textbf{Acc (\%) $\uparrow$} \\ 
        \midrule
        \multicolumn{3}{c}{Frame-based} \\ 
        \midrule
        Bittar and Garner~\cite{Bittar2022} & 0.1 & 71.7 \\
        Bittar and Garner~\cite{Bittar2022} & 3.9 & 77.4 \\
        Hammouamri et al.~\cite{hammouamri2024learning} & 0.7 & 79.8 \\
        Hammouamri et al.~\cite{hammouamri2024learning} & 2.5 & 80.7 \\
        \midrule
        \multicolumn{3}{c}{Event-based} \\ 
        \midrule
        Event-SSM~\cite{Schoene2024} & 0.1 & 85.3 \\
        Event-SSM~\cite{Schoene2024} & 0.6 & 88.4 \\
        S7~\cite{soydan2024s7selectivesimplifiedstate} & 0.6 & 88.2 \\
        STREAM (ours) & 0.2 & 86.3 \\
        \bottomrule
    \end{tabular}
    \caption{
        Comparison of STREAM to the state-of-the-art on the Spiking Speech Commands dataset \cite{Cramer2022}.
        }    \label{tab:ssc}
\end{table}
\begin{table*}[htb!]
    \centering
    \begin{tabular}{cccc S[table-number-alignment=center,table-format=2.1+-1.1]  ccc S[table-number-alignment=center,table-format=2.1+-1.1]}
    \toprule
        \multirow{3}{*}{Model} & \multicolumn{4}{c}{DVS128 Gestures} & \multicolumn{4}{c}{Spiking Speech Commands} \\
        \cmidrule{2-9}
        & \multirow{2}{*}{$t_k - t_{k-1}$} & \multicolumn{2}{c}{$\mathrm{softplus} + \mathrm{Linear}$} & {Validation} & \multirow{2}{*}{$t_k - t_{k-1}$} & \multicolumn{2}{c}{$\mathrm{softplus} + \mathrm{Linear}$} & {Validation} \\
        & & $\Delta$ & $\Gamma$ & {Accuracy} & & $\Delta$ & $\Gamma$ & {Accuracy} \\
        \midrule
                    Mamba & \ding{55} & \ding{51} & \ding{51} & 98.6+-0.7 & \ding{55} & \ding{51} & \ding{51} & 86.8+-0.1\\
        \midrule
                    & \ding{51} & \ding{55} & \ding{55} & 99.2 +- 0.3 & \ding{51} & \ding{55} & \ding{55} & 87.9 +- 0.1\\
        STREAM      & \ding{51} & \ding{55} & \ding{51} & 98.8 +- 0.2 & \ding{51} & \ding{55} & \ding{51} & 87.9 +- 0.3\\
        (variants)  & \ding{51} & \ding{51} & \ding{55} & 98.8 +- 0.6 & \ding{51} & \ding{51} & \ding{55} & 87.8 +- 0.2\\
                    & \ding{51} & \ding{51} & \ding{51} & 98.6 +- 1.2 & \ding{51} & \ding{51} & \ding{51} & 87.9 +- 0.3\\                  
        \bottomrule
    \end{tabular}
    \caption{
        STREAM versus Mamba~\cite{gu2024mambalineartimesequencemodeling} on the DVS128 Gestures~\cite{Amir2017} and Spiking Speech Commands~\cite{Cramer2022} datasets.
        STREAM directly feeds the time coordinate of irregularly spaced events as a parameter to the state-space model as described in \secref{STREAM}, which is indicated by column $t_k - t_{k-1}$.
        Different variants of STREAM might apply linear transformations and Softplus activations similar to Mamba to adjust the time scales of \(e^{\mathbf{A}\Delta_k}\) (columns $\Delta$) or the factor of $\Gamma_k$ in equation~\eqref{eq:stream2} (columns $\Gamma$). 
    }
    \label{tab:ablation}
\end{table*}
In contrast to most prior works, our method directly operates on the stream of events as visualized in \figref{event-based}.
The stream is encoded as a sequence of tokens $\func{t_i, \bfu_i}$, where each pixel of the event-based camera is uniquely encoded by a separate token, similar to how text is represented as tokens in modern language modeling~\cite{radford2019language}.
This sequence is directly processed by a STREAM module.
To accommodate large streams of up to \num{131072} events per sample in GPU memory, 
the output of the first STREAM module is subsampled by a factor of \num{8} or \num{16} depending on the task.
This way, every event is processed by the neural network at least once directly, in contrast to previous methods that subsampled the raw stream before presenting it to a neural network~\cite{Sekikawa2019, Schaefer2022}.
To create a hierarchical architecture, the sequence is subsampled a second time after half the number of layers.
The state dimension is increased by a factor of \num{2} upon the second subsampling.
Our hierarchical subsampling architecture is visualized in \figref{event-based}.
We apply average pooling along the sequence dimension before computing the class labels.
To improve generalization, we implement a set of geometric data augmentations and a variant of CutMix~\cite{Yun2019cutmix} that directly mixes event streams~\cite{Schoene2024}.

\subsubsection{Event Stream Results}
We evaluate STREAM on two popular event-based datasets.
The DVS128 Gestures dataset~\cite{Amir2017} consists of \num{1342} recordings of 10 distinct classes of hand gestures and an additional class of arbitrary other gestures.
While this dataset features only a relatively small number of samples for 11 hand gesture categories recorded from 29 subjects, the total number of events in this dataset adds up to about \SI{390}{\million} events.
Our model has an initial model dimension of $n=32$ and a state-space dimension of $m=32$, which is expanded by a factor of \num{2} upon subsampling.
We deploy a total of \num{6} STREAM blocks, which by the notation of \figref{event-based} corresponds to $L=2$.
We train with a batch size of \num{32} on sequences of \num{65536} events such that the total input sequences created by CutMix are up to \num{131072} events long, and subsample by a factor of \num{16}.
Evaluation is conducted on the full sequences of up to \SI{1.5}{\million} events.
Remarkably, our fully event-based model sets a new state-of-the-art on the DVS128-Gestures dataset as reported in \tabref{gestures}, improving over both frame-based and event-based references.
While \cite{Pei2024} reported \SI{100}{\percent} accuracy on 10 out of the 11 classes, omitting the `others' class, we for the first time present a model that can reach a maximum accuracy of \SI{100}{\percent} on all 11 classes of the dataset.

In addition to event-based vision, we evaluate STREAM on an event-based audio classification task. 
The Spiking Speech Commands dataset~\cite{Cramer2022} contains more than \num{100000} samples converted from the original Speech Commands dataset~\cite{warden2018speechcommandsdatasetlimitedvocabulary} with a median number of \num{8100} events per sample.
The model dimension is again $n=32$ with a smaller state-space dimension of $m=4$ compared to the vision model.
We deploy \num{8} STREAM blocks ($L=3$ in \figref{event-based}), and set the subsampling factor to \num{8}.
With this configuration, we report a classification accuracy of \SI{86.3}{\percent} in \tabref{ssc}, which settles between the small and large models reported in~\cite{Schoene2024}.
\subsection{Ablation Study}
\label{sec:ablation}
We compare our STREAM module against the Mamba module~\cite{gu2024mambalineartimesequencemodeling} in \tabref{ablation}.
The central difference is the explicit integration of spatio-temporal information in the recurrent operator \(e^{\mathbf{A}\Delta_k}\) with $\Delta_k = t_k - t_{k-1}$.
We additionally experiment with all combinations of applying linear transformations activated by Softplus functions to renormalize the $\Delta_k$ and $\Gamma_k$ parameters in equations \eqref{eq:stream1} and \eqref{eq:stream2}, respectively.
While the fairly small DVS128 Gestures dataset possesses high variance, STREAM consistently improves over Mamba.
This effect is stronger expressed on the larger Spiking Speech Commands dataset, which is less affected by training noise due to the larger number of samples.
Here, STREAM creates a significant margin to the Mamba baseline.

\section{Discussion}
\label{sec:discussion}
We have introduced STREAM, a sequence model for point cloud and event stream data, which achieves competitive results on a range of benchmarks from point-cloud classification to event-based vision. 
Prior efforts of unifying these modalities transferred point cloud models to event streams, or ignore the spatio-temporal structure of either modality by applying the default Mamba model to a sequential view of the data.
In contrast, our model exploits the dynamics of state-space models, a temporal process at first sight, to encode spatial geometric information into the models parameterization.
This inductive bias proved valuable as STREAM improves our reference model, PointMamba, when trained from scratch on point cloud datasets such as ModelNet40 and ScanObjectNN.
By design, STREAM applies to spatio-temporal modalities such as event-based vision.
We operate on the stream of events without collapsing events into frames and without using 2D convolutions.
This processing paradigm achieves \SI{100}{\percent} classification accuracy on all 11 classes of the event-based DVS128 Gestures dataset, a result that has so far only been achieved on the 10 predetermined classes by~\cite{Pei2024}.
A practically relevant property of STREAM is that it allows asynchronous inference on streams of points or events recorded from LiDAR sensors or event-based cameras.
We, therefore, expect STREAM to be well suited for sensor fusion of these modalities.
Previous work has shown that SSMs can benefit significantly from self-supervised pre-training on larger datasets~\cite{liang2024pointmambasimplestatespace}, which we will explore for our method in future research.

\section*{Acknowledgments}
MS is supported with funds from Bosch-Forschungsstiftung im Stifterverband.
YB and KB were funded by the German Academic Exchange Service (DAAD) under the funding programme WISE (57698568).
DK is funded by the German Federal Ministry of Education and Research (BMBF) within the project EVENTS (16ME0733).
KKN is funded by the German Federal Ministry of Education and Research (BMBF) within the KI-ASIC project (16ES0996).
CM receives funding from the German Research Foundation (DFG, Deutsche Forschungsgemeinschaft) as part of Germany’s Excellence Strategy – EXC 2050/1 – Project ID 390696704 – Cluster of Excellence “Centre for Tactile Internet with Human-in-the-Loop” (CeTI) of Technische Universität Dresden.
This work was partially funded by the German Federal Ministry of Education and Research (BMBF) and the free state of Saxony within the ScaDS.AI center of excellence for AI research.
The authors gratefully acknowledge the computing time made available to them on the high-performance computer at the NHR Center of TU Dresden. This center is jointly supported by the Federal Ministry of Education and Research and the state governments participating in the NHR (www.nhr-verein.de/unsere-partner).
The authors gratefully acknowledge the Gauss Centre for Supercomputing e.V. (www.gauss-centre.eu) for funding this project by providing computing time on the GCS Supercomputer JUWELS~\cite{JUWELS} at Jülich Supercomputing Centre (JSC).
{
    \small
    \bibliographystyle{ieeenat_fullname}
    \bibliography{main}
}

\clearpage
\setcounter{page}{1}
\maketitlesupplementary

\section{Derivations}
\label{appendix:ssm}
This section provides complete derivations of equations \eqref{eq:kernel} and \eqref{eq:recurrent-ssm}.
We restrict to the case of single-input single-output (SISO) state-space models with multi-dimensional state.
Multi-input multi-output (MIMO) formulations like S4~\cite{Gu2022} or Mamba~\cite{gu2024mambalineartimesequencemodeling} can be obtained by creating $n$ parallel instances of the SISO model and mixing the input and output components with linear layers.

Our motivation was to reason about spatial relationships between input pulses.
Therefore, we defined an interaction kernel $\bfPhi(\bfx, \bfx^\prime)$ that models pairwise interactions.
A complete view of all possible interactions with the point $\bfx_k$ is given by 
\begin{align}
    \bfy\func{\bfx_k} 
    &= \int \bfPhi\func{\bfx_k, \bfx^\prime} \bfu\func{\bfx^\prime} \mathrm{d}\bfx^\prime \label{eq:appendix-integral-operator} \\
    &= \sum_{i=0}^N \bfPhi\func{\bfx_k, \bfx_i} \bfu_i\,. \label{eq:appendix-integral-operator-dirac}
\end{align}
We will now derive a parameterization of the interaction kernel $\bfPhi$ with a state-space model.

A linear time-varying state-space model acting on a scalar input function $u\func{t}$ and producing a scalar output $y\func{t}$ is defined through
\begin{align}
    \dot{\bfh}\func{t} &= \bfA\func{t} \bfh\func{t} + \bfB\func{t} u\func{t} \label{eq:appendix-ssm-h}\\
    y\func{t} &= \bfC\func{t} \bfh\func{t}\,, \label{eq:appendix-ssm-y}
\end{align}
with ${u(t), y(t)\in\mathbb{R}}$, states ${\bfh\func{t}\in\R{m}}$ and parameters ${\bfA\func{t}\in\R{m\times m}}, {\bfB\func{t}\in\R{m\times 1}}, {\bfC\func{t}\in\R{1\times m}}$.

Note that we stick to our notation introduced in \secref{methods}, denoting vectors and matrices with bold face and scalar values with regular face.

\subsection{Solution of the Linear State-space Model}
Linear ordinary differential equations have a well known analytical solution. 
We refer the reader to standard calculus textbooks.
As such, the linear dynamics of equation \eqref{eq:appendix-ssm-h} for initial value $h_0 = h\func{t_0}$ admit the analytical solution
\begin{align}
    \bfh\func{t} 
    = h_0 + \int_{t_0}^t \exp\func{\int_{t^\prime}^t\bfA\func{t^{\prime\prime}}\mathrm{d}t^{\prime\prime}} \bfB\func{t^\prime} u\func{t^\prime} \mathrm{d}t^\prime \,, \label{eq:appendix-analytic-solution}
\end{align}
which can be checked by taking the derivative of $\bfh$ and comparing it to equation \eqref{eq:appendix-ssm-h}.
Comparing equations \eqref{eq:appendix-analytic-solution}, \eqref{eq:appendix-ssm-y} and \eqref{eq:appendix-integral-operator}, we read off the kernel
\begin{align}
    \bfPhi\func{t, t^\prime} = \bfC\func{t} \exp\func{\int_{t^\prime}^t\bfA\func{t^{\prime\prime}}\mathrm{d}t^{\prime\prime}} \bfB\func{t^\prime}
\end{align}
Note that equation \eqref{eq:appendix-analytic-solution} can be evaluated on all coordinates $t > t_0$.
The two canonical options for discretizing the variables $h$ and $y$ are equidistant steps resulting in a regular grid, or using the input coordinates $t_0, \dots, t_N$.
While most state-space model works discretize on equidistant steps, \cite{Smith2023} shows on a toy task that the SSM formulation is capable of solving tasks with irregularly spaced steps as well.

Our work differs from other recent Mamba based models such as PointMamba~\cite{liang2024pointmambasimplestatespace}, Point Cloud Mamba~\cite{Tao2024pointcloudmamba}, EventMamba~\cite{ren2024rethinkingefficienteffectivepointbased}, or SpikMamba~\cite{chen2024spikmambasnnmeetsmamba} by integrating the true timings of the inputs in the following sense.

We will use $h_0 = 0$ for notational simplicity in the following.
Discretizing equation \eqref{eq:appendix-analytic-solution} on the Dirac delta coded input \eqref{eq:input} then yields
\begin{align}
    \bfh\func{t_k}
    &= \int_{t_0}^{t_k} \exp\func{\int_{t}^{t_k}\bfA\func{t^\prime}\mathrm{d}t^\prime} \bfB\func{t} \sum_{i=0}^N\delta\func{t - t_i}u_i\mathrm{d}t \nonumber\\
    &= \sum_{i=0}^{k} \exp\func{\int_{t_i}^{t_k}\bfA\func{t}\mathrm{d}t} \bfB\func{t_i} u_i \,.
\end{align}
We decompose the integral from $t_i$ to $t_k$ into the integrals from $t_{j-1}$ to $t_j$ and sum over them
\begin{align}
    \bfh\func{t_k}
    &= \sum_{i=0}^{k} \exp\func{\sum_{j=i+1}^{k}\int_{t_{j-1}}^{t_j}\bfA\func{t}\mathrm{d}t} \bfB\func{t_i} u_i \,.
\end{align}
We will further assume that $\bfA\func{t}, \bfB\func{t}, \bfC\func{t}$ are constant on the intervals $(t_{j-1}, t_j]$ for $j=i+1, \dots, k$.
Denoting ${\bfA\func{t_j} = \bfA_j}$, ${\bfB\func{t_j} = \bfB_j}$, ${\bfC\func{t_j} = \bfC_j}$, ${\bfh\func{t_k} = \bfh_k}$, and ${\Delta_j = t_j - t_{j-1}}$ we get
\begin{align}
    \bfh\func{t_k}
    &= \sum_{i=0}^{k} \exp\func{\sum_{j=i+1}^{k}\bfA_j\Delta_j} \bfB_i u_i  \\
    &= \sum_{i=0}^{k} \left(\prod_{j=i+1}^k\exp\func{\bfA_j\Delta_j}\right) \bfB_i u_i \,. \label{eq:appendix-ssm-discrete}
\end{align}
Comparing to equation \eqref{eq:appendix-integral-operator-dirac}, we read off the discrete kernel function
\begin{align}
    \bfPhi\func{t_k, t_i} = \bfC_k \left(\prod_{j=i+1}^k\exp\func{\bfA_j\Delta_j}\right) \bfB_i \,. \label{eq:appendix-kernel-discrete}
\end{align}
\begin{proposition}
    The kernels parameterized by equation \eqref{eq:appendix-kernel-discrete} contain convolution operations with rational kernels as a special case.
\end{proposition}
\begin{proof}
    Consider the linear time-invariant (LTI) case where $\bfA, \bfB, \bfC$ are constant for all $t$.
    Then 
    \begin{align}
        \prod_{j=i+1}^k\exp\func{\bfA_j\Delta_j} = \exp\func{\bfA\func{t_k - t_i}}
    \end{align}
    The proposition follows from the fundamental result from linear systems theory that LTI state-space models can represent every rational transfer function.
\end{proof}
\subsection{Recursive Computation of the Interaction Kernel}
Let's expand equation \eqref{eq:appendix-ssm-discrete} in a recursive form
\begin{align*}
    \bfh_k = \bfh\func{x_k} 
    &= \sum_{i=0}^{k} \left(\prod_{j=i+1}^ke^{\bfA_j\Delta_j}\right) \bfB_i u_i\\
    &= \bfB\bfu_k + \sum_{i=0}^{k-1} \left(\prod_{j=i+1}^ke^{\bfA_j\Delta_j}\right) \bfB_i u_i\\
    &=\bfB\bfu_k + \sum_{i=0}^{k-1} e^{\bfA_k\Delta_k} \left(\prod_{j=i+1}^{k-1}e^{\bfA_j\Delta_j}\right) \bfB_i u_i\\
    &=\bfB\bfu_k + e^{\bfA_k\Delta_k} \sum_{i=0}^{k-1}\left(\prod_{j=i+1}^{k-1}e^{\bfA_j\Delta_j}\right) \bfB_i u_i\\
    &=\bfB\bfu_k + e^{\bfA_k\Delta_k} \bfh_{k-1} \,.
\end{align*}
This concludes the derivation of our recurrent operator
\begin{align}
    \bfh_k
    =\bfB\bfu_k + e^{\bfA_k\Delta_k} \bfh_{k-1} \,. \label{eq:appendix-recurrent-ssm}
\end{align}
In line with most recent SSM works, we choose $\bfA_i \equiv \bfA$ as a diagonal matrix for all $i$.
\begin{remark}
    Equation \eqref{eq:appendix-recurrent-ssm} allows asynchronous inference on streams of incoming coordinates in $\Order{1}$ time per coordinate, or $\Order{N}$ time for the full stream of coordinates.
\end{remark}
\section{Scan}
\label{appendix:scan}
Linear time-varying systems such as the one given by equation \eqref{eq:appendix-recurrent-ssm} resemble an associative operation, with well-known time complexity of $\Order{\log N}$ \cite{blelloch1990prefix}.
The goal of this section is not to provide a proof, but to give the reader a clear idea of how irrgularly spaced sequences can be parallelized in the same way that regular SSMs can be parallelized with the Scan primitive~\cite{Smith2023, gu2024mambalineartimesequencemodeling}.
The presentation follows \cite{blelloch1990prefix}.

Consider the pair
\begin{align}
    \bfc_i = \left[\bfa_i, \bfb_i\right]
\end{align}
with the binary operator $\bullet$ defined through
\begin{align}
    \bfc_i \bullet\bfc_j 
    = \left[ \bfa_i \cdot \bfa_j, \bfa_j \cdot \bfb_i + \bfb_j\right] \,.
\end{align}
As~\cite{blelloch1990prefix} shows, the operator $\bullet$ is associative, i.e.
\begin{align}
    \left(\bfc_i \bullet \bfc_j \right) \bullet \bfc_k = \bfc_i \bullet \left( \bfc_j \bullet \bfc_k \right) \,.
\end{align}
Associative operators can be parallelized to run in $\Order{\log N}$ time on sequence of length $N$ given sufficiently many processors~\cite{blelloch1990prefix}.
\begin{figure*}[!htb]
    \centering
    \includegraphics[width=\textwidth]{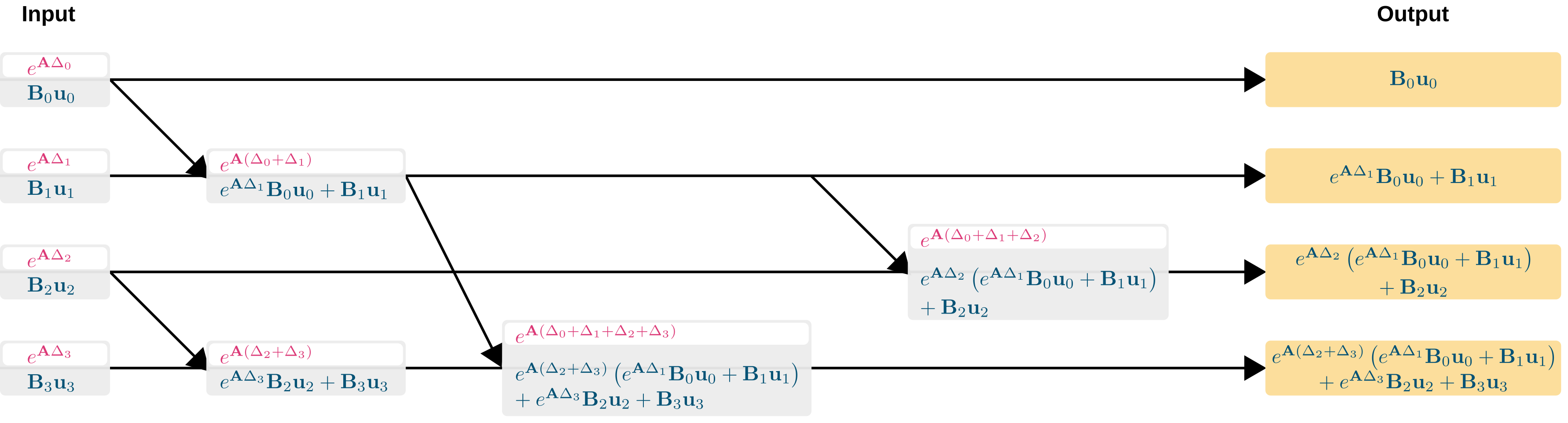}
    \caption{
        Example of a Scan operating on four inputs.
        The left most column represents the input pairs $(e^{\bfA\Delta_i}, \bfB_i\bfu_i)$, and the right most column represents the final results.
    }
    \label{fig:scan}
\end{figure*}
We use this primitive to parallelize our model.
The pairs $\bfc$ are initialized as 
\begin{align}
    \bfc_0 = \left[ e^{\bfA\Delta_0}, \bfB_0\bfu_0\right] \quad \dots \quad \bfc_N = \left[ e^{\bfA\Delta_N}, \bfB_N\bfu_N\right]
\end{align}
\Figref{scan} shows an example of how the Scan primitive computes the entire recurrence in equation \eqref{eq:appendix-recurrent-ssm} in roughly $2\log N$ steps.

\end{document}